\documentclass[letterpaper, 10 pt, conference]{ieeeconf}  

\IEEEoverridecommandlockouts                              

\usepackage{cite}
\usepackage{amsmath,amssymb,amsfonts}
\usepackage{algorithmic}
\usepackage{graphicx}
\usepackage{textcomp}
\usepackage{xcolor}
\usepackage{amsmath,amssymb,amsfonts}
\usepackage{mathtools}
\usepackage{graphicx}
\usepackage{booktabs}
\usepackage{multirow}
\usepackage{microtype}
\usepackage{cite}
\usepackage{url}
\usepackage{xcolor}
\usepackage{tikz}
\usetikzlibrary{arrows.meta,positioning,fit}
\usepackage{algorithm}
\usepackage{algorithmic}
\usepackage{graphicx}
\usepackage{subcaption}
\usepackage{xcolor}

\newcommand{{\method}}{\textsc{RefinePPO}}
\newcommand{\rar}{\textsc{IAR}}
\newcommand{\E}{\mathbb{E}}
\newcommand{\R}{\mathbb{R}}

\newtheorem{proposition}{Proposition}

\def\BibTeX{{\rm B\kern-.05em{\sc i\kern-.025em b}\kern-.08em
    T\kern-.1667em\lower.7ex\hbox{E}\kern-.125emX}}
\begin{document}

\title{\LARGE \bf
{\method}: Learning Continuous Control Policies by\\
Iterative Action Refinement
}


 \author{
 Sachini Weerasekara\textsuperscript{1},
 Sagar Kamarthi\textsuperscript{1},
 Jacqueline Isaacs\textsuperscript{1}
 \thanks{\textsuperscript{1}Northeastern University,
 Boston, MA, USA.
 \{weerasekara.s, s.kamarthi, j.isaacs\}@northeastern.edu}
 }

\maketitle

\begin{abstract}
Deep reinforcement learning (DRL) has achieved strong performance across a wide range of continuous-control problems. These continuous-control policies, however, are often defined as direct mappings from an observed state to an action or action distribution, requiring a single feed-forward network to construct an optimal control decision in one pass. While effective, this formulation leaves little opportunity for the policy to reconsider or progressively improve an action once an initial prediction has been formed. In this work, we explore an alternative approach: rather than learning only to directly predict an action, can a policy learn to iteratively improve one, and can this iterative process provide advantages during policy learning? We introduce \emph{Iterative Action Refinement} (\rar), an iterative action-construction method that constructs control actions through a sequence of learned residual corrections. Starting from an initial proposal, a shared refinement network repeatedly conditions on the observed state and the current action proposal, allowing each refinement step to revise the action constructed by preceding steps. The final refined proposal is then used to determine the action executed by the agent. We integrate this iterative action-construction mechanism with Proximal Policy Optimization (PPO), yielding {\method}. We evaluate {\method} across 14 benchmark control tasks, complemented by controlled ablations of refinement depth and update schedules and analyses aimed at understanding why iterative refinement is effective. Across these environments, {\method} matches or exceeds the performance of standard PPO while demonstrating faster convergence on several tasks.
\end{abstract}

\section{Introduction}

Deep reinforcement learning has enabled solving increasingly complex continuous-control problems in domains such as robotics~\cite{lillicontinuous,levine2016end}, and autonomous driving~\cite{kiran2021deep, 10411826}. However, much of this progress has focused on improving how policies are optimized: developing more stable objectives~\cite{schulman2017proximal}, improving exploration~\cite{haarnoja2018soft}, reducing variance~\cite{schulman2015high}, better generalization~\cite{jayawardana2025intersectionzoo} and making better use of collected experience~\cite{schaul2015prioritized}. Comparatively less attention has been given to a more basic question about the computation performed by the policy itself: \emph{how should a neural policy construct an action from the current state?}

The dominant approach is straightforward. Given an observation, a neural network predicts an optimal action or the parameters of an optimal action distribution in a single forward pass, which is then sent to the agent for execution. Many successful continuous-control methods share this design, which offers important practical advantages: it is simple, intuitive, and has already shown strong success across domains. Proximal Policy Optimization (PPO)~\cite{schulman2017proximal}, for example, commonly represents a continuous policy as a Gaussian whose mean is predicted directly from the current state using a feedforward network, and the action is then sampled from the resulting distribution~\cite{huang2022cleanrl}.

An alternative is to treat action generation not as a one-shot prediction, but as an iterative process in which an initial action proposal is progressively improved before execution. Rather than asking a network to produce the optimal action at once, the policy can learn a correction mechanism that repeatedly revises its current proposal based on both the observed state and what has already been constructed.

This observation motivates the question we study in this work: \emph{can continuous-control policies benefit from learning how to improve an action proposal rather than only learning how to predict the optimal action directly?}. Instead of requiring the policy to resolve the optimal control decision in one computation, we allow the policy to construct the optimal control policy progressively. Each computation begins with the current proposal, determines how that proposal should be changed given the state, and passes the revised proposal to the next computation. Such a process gives later computations access to what earlier computations have already constructed and turns action generation from a one-shot prediction into a learned sequence of corrections. 

We introduce \emph{Iterative Action Refinement} (\rar), an iterative action-construction method based on this idea. Starting from an initial action proposal, {\rar} repeatedly applies a shared residual network that conditions on both the observed state and the current proposal. Given state $s$ and proposal $m_k$, the network $f_\theta$ predicts a correction $f_\theta(s,m_k)$ and updates the proposal according to
\begin{equation}
m_{k+1}=m_k+c_{k+1}f_\theta(s,m_k), \qquad m_0=0
\label{eq:refine}
\end{equation}
for $K$ refinement steps. Each intermediate $m_k$ is therefore a latent action proposal rather than an action executed in the environment. After the refinement process is complete, we construct the stochastic policy,
\begin{equation}
a\sim\mathcal{N}\left(m_K,\operatorname{diag}(\sigma^2)\right)
\end{equation}

\noindent where $\mathcal{N}$ denotes a Gaussian distribution and $\sigma$ denotes its vector of action-wise standard deviations. The final refined proposal $m_K$ serves as the Gaussian policy mean, while exploration retains the standard stochastic form used by the underlying continuous-control policy.

This alternative formulation provides several potential advantages for continuous-control policies. First, iterative refinement decomposes the state-to-action mapping into a sequence of conditional corrections, allowing the policy to progressively construct a control decision rather than requiring a single computation to produce the optimal action. Second, because each refinement step conditions on the current action proposal, later computations can explicitly account for decisions made by earlier steps and adjust them in the context of the full action being constructed. This self-conditioning may be particularly useful in high-dimensional control problems, where effective behavior requires coordination across multiple action dimensions. Third, the refinement network is shared across steps, allowing the policy to perform additional computation without introducing a separate set of parameters for each refinement stage. The refinement depth $K$ therefore provides a direct mechanism for trading additional policy computation for progressively deeper action refinement. Finally, weight sharing encourages the policy to learn a reusable correction rule that is applied across different intermediate action proposals, rather than associating each stage of computation with a separate transformation. This may promote a more structured and generalizable action-construction process, while potentially making useful control behaviors easier to learn and improving learning efficiency.

We integrate {\rar} with PPO, yielding {\method}. The integration is deliberately minimal: refinement changes how the policy mean is computed, but does not alter the surrounding policy-gradient algorithm. The PPO clipped surrogate objective, probability ratio, critic, advantage estimator, and entropy formulation remain unchanged. During policy evaluation, the actor performs $K$ deterministic refinement steps before sampling the environment action; during optimization, gradients propagate through the complete refinement chain. Consequently, {{\method}} introduces iterative computation \emph{within} each policy evaluation while PPO continues to perform policy optimization \emph{across} collected experience.

We evaluate {{\method}} on 14 continuous control tasks from classic control, Box2D and MuJoCo in Gymnasium complemented by controlled ablations of refinement depth, update schedules and analyses aimed at understanding why iterative refinement is effective. We compare against standard PPO~\cite{schulman2017proximal}. Across the environments, {\method} achieves performance that is competitive with or exceeds the standard PPO baseline and exhibits faster convergence on several tasks. These results indicate that iterative action construction can provide a useful alternative to conventional one-pass action prediction, while also showing that the benefit depends on the choice of refinement depth and update dynamics.

To summarize, our contributions are threefold. First, we introduce Iterative Action Refinement (\rar), a simple iterative action-construction mechanism in which a shared residual network progressively refines an action proposal. Second, we characterize the refinement process within PPO as {{\method}} without modifying the underlying policy-gradient objective. Third, we provide an empirical evaluation across 14 benchmark continuous control tasks together with controlled ablations of refinement depth and update schedules and analyses aimed at understanding why iterative refinement is effective. Taken together, these results motivate iterative action refinement as a simple but distinct computational bias for continuous-control policies: rather than requiring a policy to construct its final action in a single pass, the actor can learn a reusable process for progressively improving its own action proposals.

\textbf{Remark}: While we focus on PPO in this work, {\rar} is not inherently tied to PPO. Extending {\rar} to other methods is a natural direction for future work.

\section{Related Work}

In this section, we review prior work most closely related to \rar{}, including policy optimization for continuous control, residual reinforcement learning, iterative computation and learned optimization, and planning and optimization within policies.

\subsection{Policy Optimization for Continuous Control}
Policy-gradient methods optimize expected return by directly differentiating a parameterized stochastic policy~\cite{sutton1999policy}. Trust Region Policy Optimization (TRPO) stabilizes this process by constraining policy updates through a KL-divergence trust region~\cite{schulman2015trust}, while Proximal Policy Optimization (PPO) replaces the constrained optimization with a clipped surrogate objective that limits excessively large policy updates~\cite{schulman2017proximal}. PPO has consequently become a widely used baseline for continuous-control and robotic locomotion tasks. Our approach leaves the PPO objective and optimization procedure unchanged and instead modifies how the policy constructs the mean of its action distribution.

Other continuous-control methods, including DDPG~\cite{lillicrap2015continuous}, TD3~\cite{fujimoto2018addressing}, and SAC~\cite{haarnoja2018soft}, differ in their optimization objectives and exploration mechanisms but similarly rely on neural actors that directly construct actions or action-distribution parameters from the current state. In this work, we focus exclusively on integrating {\rar} with PPO in order to study the effect of iterative action refinement within a controlled policy-optimization setting. Investigating how {\rar} can be integrated with other policy-gradient and actor--critic methods, and whether its benefits extend across different optimization frameworks, remains an important direction for future work.

\subsection{Residual Reinforcement Learning}
Residual learning has also been explored directly in reinforcement learning, particularly through residual policies that learn corrections to an existing controller or policy. In residual reinforcement learning, the learned policy typically produces an additive correction to an action supplied by a fixed controller, allowing prior control knowledge to be combined with learned behavior~\cite{johannink2019residual, silver2018residual,weerasekara2026prototype,weerasekara2025cellclique} and later work exploring more additions such as multi residual task learning~\cite{jayawardana2024generalizing} and mixture of experts in the context of residual reinforcement learning~\cite{jayawardana2025multi}. Although {\rar} also uses additive residual corrections, the role of the residual is fundamentally different. {\rar} does not correct the output of an external controller. Instead, a single learned policy repeatedly corrects its own intermediate action proposal within one decision step. The residual structure therefore operates \emph{inside} the policy's action-construction process rather than between a learned policy and a pre-existing controller.

\subsection{Iterative Computation and Learned Optimization}

The iterative structure of {\rar} is related to methods that construct predictions through repeated updates. Residual networks build representations through sequences of incremental transformations~\cite{he2016deep}, while neural ODEs connect such residual updates to continuous-time dynamics~\cite{chen2018neural}. Learned optimizers predict updates to candidate solutions rather than directly producing final solutions~\cite{andrychowicz2016learning}, and iterative amortized inference similarly improves an initial estimate through successive learned corrections~\cite{marino2018iterative}. More recently, looped Transformers use weight-tied iterations to progressively refine representations~\cite{giannou2023looped}. These approaches motivate the broader principle underlying {\rar}: a difficult prediction can be represented as a learned refinement process rather than a single direct mapping.

Deep equilibrium models (DEQs) extend repeated weight-tied computation by defining representations through fixed points of learned transformations~\cite{bai2019deep,weerasekara2025improvements}. {\rar} instead performs a finite, explicit number of refinement steps and backpropagates through the resulting computation normally. Nevertheless, the fixed-point perspective provides a useful interpretation of repeated action refinement and, under appropriate contraction conditions, of how successive proposals approach a solution.

\subsection{Planning and Optimization Within Policies}

Iterative action construction is also related to approaches that perform planning or optimization as part of decision-making. Model-predictive control and model-based reinforcement learning methods can optimize candidate action sequences using a model of the system dynamics~\cite{chen2018neural}, while differentiable planning methods embed structured planning computations within trainable neural architectures~\cite{tamar2016value,weerasekara2022trends, weerasekara2024reinforcement}. \rar{} differs from these approaches in that it requires neither a dynamics model nor an explicit planning or action-space optimization procedure at execution time. Instead, the refinement rule is learned end-to-end through the policy objective, and inference consists of repeatedly applying the learned residual network to its current action proposal. In this sense, \rar{} can be viewed as amortized iterative computation in action space: policy learning acquires a reusable update rule for constructing actions, rather than solving a new model-based planning or optimization problem at every decision step.

\section{Preliminaries}

We consider an infinite-horizon discounted Markov decision process (MDP) $\mathcal{M}=(\mathcal{S},\mathcal{A},P,r,\gamma)$, where $\mathcal{S}$ and $\mathcal{A}\subseteq\R^{d_a}$ denote the state and continuous action spaces, respectively, $d_a$ is the action dimension, $P(s'\mid s,a)$ is the transition kernel, $r(s,a)$ is the reward function, and $\gamma\in[0,1)$ is the discount factor. A stochastic policy $\pi_\theta(a\mid s)$, parameterized by $\theta$, induces a trajectory $\tau=(s_0,a_0,s_1,a_1,\ldots)$ and is optimized to maximize the expected discounted return
\begin{equation}
J(\theta)=\E_{\tau\sim\pi_\theta}\left[\sum_{t=0}^{\infty}\gamma^t r(s_t,a_t)\right]
\end{equation}

For continuous actions, a standard PPO actor parameterizes the policy as a diagonal Gaussian,
\begin{equation}
\pi_\theta(a\mid s)=\mathcal{N}\left(a;\mu_\theta(s),\operatorname{diag}(\sigma_\theta^2)\right) 
\label{eq:gaussian}
\end{equation}
where $\mu_\theta(s)\in\R^{d_a}$ is the state-dependent mean produced by the actor network and $\sigma_\theta\in\R_{>0}^{d_a}$ is the vector of action-wise standard deviations. In the standard PPO parameterization considered here, $\mu_\theta(s)$ is computed in a single forward pass, while $\log\sigma_\theta$ is a learned vector.

For PPO, let $\theta_{\mathrm{old}}$ denote the behavior-policy parameters and $\hat A_t$ a generalized advantage estimate~\cite{schulman2015high}. The likelihood ratio is
\begin{equation}
r_t(\theta)=
\frac{\pi_\theta(a_t\mid s_t)}
{\pi_{\theta_{\mathrm{old}}}(a_t\mid s_t)}
\end{equation}
and the clipped actor objective is
\begin{equation}
L^{\mathrm{CLIP}}(\theta)=
\E_t\left[
\min\left(
r_t(\theta)\hat A_t,
\operatorname{clip}(r_t(\theta),1-\epsilon,1+\epsilon)\hat A_t
\right)
\right]
\label{eq:ppo}
\end{equation}
where $\epsilon>0$ denotes the PPO clipping parameter.

\section{Problem Formulation}
\label{sec:problem}

A standard continuous-control actor constructs the policy mean through a direct mapping
\begin{equation}
\mu_\theta:\mathcal{S}\rightarrow\R^{d_a}
\end{equation}
such that the complete action proposal is produced from the current state in a single computation. We consider a more general formulation in which action construction is itself a sequential computation.

For a fixed state $s\in\mathcal{S}$, let $m_k\in\R^{d_a}$ denote an intermediate action proposal and consider a state-conditioned transition operator
\begin{equation}
m_{k+1}=G_\theta(s,m_k),
\qquad k=0,\ldots,K-1
\label{eq:generic_dynamics}
\end{equation}
Starting from an initial proposal $m_0$, repeated application of $G_\theta$ generates an internal trajectory in action space,
\begin{equation}
m_0\rightarrow m_1\rightarrow\cdots\rightarrow m_K
\end{equation}
The environment state $s$ remains fixed throughout this internal computation, while the action proposal evolves. Only the terminal proposal determines the policy mean,
\begin{equation}
\mu_{\theta,K}(s)\triangleq m_K(s)
\end{equation}
whereas the intermediate proposals are latent computations and are never executed in the environment.

This viewpoint casts action construction as a finite-horizon dynamical system in action space. The problem is then to learn dynamics that transform an initial proposal into a useful policy mean through successive state-conditioned updates. Equation~\eqref{eq:generic_dynamics} deliberately leaves the form of these dynamics unspecified. We seek a simple realization that allows later computations to revise earlier action proposals, reuses the same learned transformation across steps, and can be incorporated into a stochastic continuous-control policy without changing its underlying optimization objective.

\section{Method: Iterative Action Refinement}
\label{sec:method}

\begin{figure*}
    \centering
    \includegraphics[width=1\linewidth]{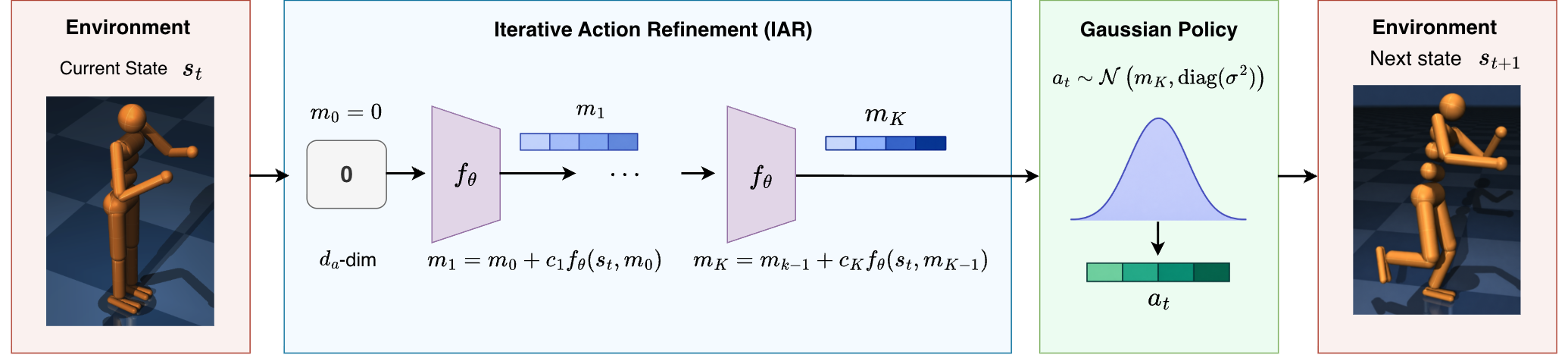}
    \caption{Overview of {\method}. Given the current environment state $s_t$, the policy initializes an action proposal $m_0=\mathbf{0}$ and progressively refines it for $K$ steps using the state-conditioned residual network $f_\theta$. The state $s_t$ remains fixed throughout refinement, while each update conditions on the current action proposal $m_k$. The final proposal $m_K$ defines the mean of the Gaussian policy, from which a single action $a_t$ is sampled and executed in the environment, producing the next state $s_{t+1}$. Intermediate proposals are latent computations and are never executed in the environment.}
    \label{fig:placeholder}
\end{figure*}

\subsection{Iterative Refinement Dynamics}

We instantiate the action-space dynamics in Eq.~\eqref{eq:generic_dynamics} using \emph{Iterative Action Refinement} (\rar). Given a state $s$ and current proposal $m_k$, a shared refinement network $f_\theta$ predicts a correction in action space. Starting from the zero vector $m_0=\mathbf{0}\in\R^{d_a}$, the proposal evolves according to
\begin{equation}
m_{k+1}
=
m_k+c_{k+1}f_\theta(s,m_k),
\qquad k=0,\ldots,K-1
\label{eq:refine}
\end{equation}
where refinement scheduler $c_{k+1}>0$ controls the magnitude of the corresponding update. The same parameters $\theta$ are shared across all refinement steps.

After $K$ refinements, the terminal proposal defines the mean of the stochastic policy,
\begin{equation}
\pi_{\theta,K}(a\mid s)
=
\mathcal{N}\left(
a;m_K(s),\operatorname{diag}(\sigma_\theta^2)
\right)
\label{eq:rarpolicy}
\end{equation}
Thus, the refinement trajectory is deterministic conditioned on $s$ and the policy parameters, and stochasticity is introduced only after refinement through the final Gaussian policy.

Under this parameterization, the abstract transition operator in Eq.~\eqref{eq:generic_dynamics} takes the form
\begin{equation}
G_{\theta,k}(s,m)
=
m+c_{k+1}f_\theta(s,m)
\end{equation}
Because each correction depends on the current proposal, later refinements explicitly condition on what earlier refinements have already constructed.

\begin{proposition}[One-step refinement]
For $K=1$, $c_1=1$, and $m_0=\mathbf{0}\in\R^{d_a}$, the {\rar} policy mean satisfies
$ \mu_{\theta,1}(s)=f_\theta(s,\mathbf{0})$.
Since the proposal input is fixed, $\mu_{\theta,1}$ is functionally a direct mapping from state to action mean.
\end{proposition}

\noindent \textit{Proof:} Applying Eq.~\eqref{eq:refine} once gives $m_1=f_\theta(s,\mathbf{0})$
Since the terminal proposal defines the policy mean, $\mu_{\theta,1}(s)=m_1=f_\theta(s,\mathbf{0})$, which depends only on $s$.

Thus, $K=1$ contains no iterative refinement, while $K>1$ allows each correction to condition on proposals generated by previous applications of the same network.

\subsection{Integration with PPO}

Integrating \rar{} with PPO requires only replacing the direct computation of the Gaussian mean with the refinement process. During rollout, the actor computes $m_K$ and samples a single action from Eq.~\eqref{eq:rarpolicy}. During a PPO update, $m_K$ is recomputed under the current parameters and the stored action is evaluated under the resulting policy. The likelihood ratio is therefore
\begin{equation}
r_t(\theta)
=
\frac{\pi_{\theta,K}(a_t\mid s_t)}
{\pi_{\theta_{\mathrm{old}},K}(a_t\mid s_t)}
\end{equation}
and is used directly in the standard PPO objective of Eq.~\eqref{eq:ppo}. The clipped objective, value loss, advantage estimator, and entropy term are otherwise unchanged. The full algorithm is given in Algorithm~\ref{alg:refineppo}.

\begin{algorithm}[t]
\caption{{\method}}
\label{alg:refineppo}
\begin{algorithmic}[1]
\REQUIRE State $s$, refinement depth $K$, coefficients $\{c_k\}_{k=1}^{K}$, refinement network $f_\theta$, log standard deviation $\log\sigma_\theta$
\STATE $m \leftarrow \mathbf{0}$
\FOR{$k=1,\ldots,K$}
    \STATE $m \leftarrow m + c_k f_\theta(s,m)$
\ENDFOR
\STATE $\pi_{\theta,K}(\cdot\mid s)
\leftarrow
\mathcal{N}\!\left(m,\operatorname{diag}(\sigma_\theta^2)\right)$
\IF{training}
    \STATE Sample $a \sim \pi_{\theta,K}(\cdot\mid s)$
    \STATE Update $\theta$ using the standard PPO objective
\ELSE
    \STATE Set $a \leftarrow m$ \COMMENT{deterministic evaluation}
\ENDIF
\end{algorithmic}
\end{algorithm}

\section{Why Iterative Action Refinement Can Work}
\label{sec:understanding}

In this section, we use a simple toy task to illustrate one potential utility of {\rar} and provide intuition for why it can be effective. Specifically, to demonstrate the benefits of iterative refinement, consider a two-dimensional action $a=(a_1,a_2)$ whose optimal components satisfy
\begin{equation}
a_1^\star = \phi(s),
\qquad
a_2^\star = \psi(s,a_1^\star)
\label{eq:toy_action}
\end{equation}
Here, the appropriate value of $a_2$ depends on the decision made for $a_1$, representing a simple form of coordination between action dimensions. Such dependencies arise naturally in real-world tasks: in robotic manipulation, the appropriate gripper orientation may depend on the chosen end-effector position; in autonomous driving, the appropriate steering angle may depend on the vehicle's speed; and in locomotion, the placement of one foot may depend on the position and motion of the other limbs.

A conventional PPO policy predicts both components directly from the state,
\begin{equation}
\mu_\theta(s)
=
\begin{bmatrix}
\mu_{\theta,1}(s) \\
\mu_{\theta,2}(s)
\end{bmatrix}
\approx
\begin{bmatrix}
\phi(s) \\
\psi(s,\phi(s))
\end{bmatrix}
\label{eq:toy_ppo}
\end{equation}
Thus, although a sufficiently expressive network can represent the optimal policy, the dependency between $a_1$ and $a_2$ must be learned implicitly within a single state-to-action computation, which may require more samples to reliably capture the coordination between action dimensions.

In contrast, {\rar} can resolve this dependency progressively across refinement steps. Recall that action generation always begins from the fixed initialization $m_0=\mathbf{0}$. For $K=2$, the first refinement receives $(s,\mathbf{0})$ and can construct an intermediate action estimate such as
\begin{equation}
m_1
=
m_0 + f_\theta(s,m_0)
=
f_\theta(s,\mathbf{0})
\approx
\begin{bmatrix}
\phi(s) \\
0
\end{bmatrix}
\label{eq:toy_refinement_1}
\end{equation}
The second refinement then receives $(s,m_1)$. Unlike the first step, it therefore has explicit access to the intermediate decision $m_{1,1}\approx \phi(s)$ and can use this information when refining the action:
\begin{equation}
m_2
=
m_1 + f_\theta(s,m_1)
\approx
\begin{bmatrix}
\phi(s) \\
\psi(s,m_{1,1})
\end{bmatrix}
\approx
\begin{bmatrix}
\phi(s) \\
\psi(s,\phi(s))
\end{bmatrix}
\label{eq:toy_refinement_2}
\end{equation}
Thus, rather than requiring the composition $\psi(s,\phi(s))$ to be produced entirely within a single state-to-action computation, refinement provides a computational path in which $\phi(s)$ can first be represented in the intermediate action and then made directly available when constructing the dependent component.

Importantly, this example does not imply that {\rar} is constrained to refine one action dimension at a time, nor that a conventional PPO policy cannot represent the same mapping. Both policies may be sufficiently expressive to represent the optimal action. The distinction is instead in the structure of the computation. Later in experimental results, we will provide evidence that these behaviors are likely happening in {\rar}. 

\section{Experimental Setup}

\subsection{Benchmarks and Baselines}

We perform experiments on 14 continuous control tasks from classic control, Box2D and MuJoCo environments in Gymnasium. The primary baseline is the standard PPO actor~\cite{schulman2017proximal}. 

\subsection{Implementation Details}

We use CleanRL~\cite{huang2022cleanrl} implementation of continuous action PPO and further modify it to implement {\method}. The baseline PPO implementation is same as the CleanRL implementation. Unless otherwise noted, training uses 6 million environment steps for each envionrment. We independently performed hyperparameter sweep for both baseline PPO and {{\method}}. As fixed hyperparameters, the implementation uses orthogonal initialization with $\tanh$ nonlinearities. The critic contains two 64-unit hidden layers. The baseline actor contains two 128-unit hidden layers, while the refinement field uses two 128-unit hidden layers after concatenating the normalized observation and current action proposal. All results are reported as average over 5 independent seeds. 

\subsection{Evaluation Protocol}

At fixed environment-step intervals during training, we evaluate the deterministic policy using $a=m_K(s)$ for {{\method}} and $a=m(s)$ for PPO, both without exploration noise. For each training seed, we select the checkpoint achieving the highest evaluation return during training and use this checkpoint for final evaluation. The selected checkpoint is then evaluated over a ten evaluation episodes, and the resulting returns are averaged to obtain the final performance for that seed. We then average these per-seed final returns across all training seeds and report the mean and standard deviation across seeds.

\subsection{Refinement Depth}

The refinement depth $K$ determines the number of residual corrections applied before constructing the final action distribution in {{\method}}. We consider $K\in\{1,2,4,8\}$ to evaluate the effect of increasing refinement depth on policy performance. Here, $K=1$ represents the one-step boundary case with no iterative refinement, while larger values of $K$ provide progressively more refinement steps.

\subsection{Refinement Schedules}

The coefficients $\{c_k\}_{k=1}^{K}$ in Eq.~\eqref{eq:refine} control the contribution of each residual update. We consider three schedules:
\begin{equation}
\text{Uniform: } c_k=1,\quad
\text{Inverse: } c_k=\frac{1}{k},\quad
\text{Average: } c_k=\frac{1}{K}
\end{equation}
The uniform schedule applies every correction at full scale, whereas the inverse schedule progressively reduces later updates. For the average schedule, $\sum_{k=1}^{K} c_k = 1$. 

\section{Experimental Results}

\begin{table*}[t]
\centering
\caption{Evaluation performance of PPO and \method{} across 14 continuous-control
environments. Results report the mean return $\pm$ standard deviation for the
best-performing hyperparameter configuration of each method.}
\label{tab:performance}

\resizebox{\textwidth}{!}{%
\begin{tabular}{|l|c|c|c|c||l|c|c|c|c|}
\hline
\textbf{Environment} & \textbf{Dim.} & \textbf{PPO} & \textbf{\method} & \textbf{Improv. (\%)} &
\textbf{Environment} & \textbf{Dim.} & \textbf{PPO} & \textbf{\method} & \textbf{Improv. (\%)} \\
\hline

Ant-v5
& 8
& $4101 \pm 681$
& $\mathbf{4370 \pm 472}$
& $\mathbf{+6.6\%}$
&
Pendulum-v1
& 1
& $\mathbf{-770 \pm 445}$
& $-833 \pm 73$
& $-8.2\%$
\\

HalfCheetah-v5
& 6
& $3034 \pm 1803$
& $\mathbf{5686 \pm 2640}$
& $\mathbf{+87.4\%}$
&
MountainCarContinuous-v0
& 1
& $93 \pm 0.2$
& $\mathbf{94 \pm 0.1}$
& $\mathbf{+1.1\%}$
\\

Hopper-v5
& 3
& $2720 \pm 826$
& $\mathbf{3120 \pm 287}$
& $\mathbf{+14.7\%}$
&
LunarLanderContinuous-v3
& 2
& $269 \pm 20$
& $\mathbf{281 \pm 8}$
& $\mathbf{+4.5\%}$
\\

Walker2d-v5
& 6
& $5042 \pm 929$
& $\mathbf{5810 \pm 354}$
& $\mathbf{+15.2\%}$
&
BipedalWalker-v3
& 4
& $289 \pm 4.2$
& $\mathbf{291 \pm 4}$
& $\mathbf{+0.7\%}$
\\

Swimmer-v5
& 2
& $123 \pm 8$
& $\mathbf{197 \pm 134}$
& $\mathbf{+60.2\%}$
&
BipedalWalkerHardcore-v3
& 4
& $-28 \pm 47$
& $\mathbf{-4 \pm 59}$
& $\mathbf{+85.7\%}$
\\

Pusher-v5
& 7
& $-44.0 \pm 10$
& $\mathbf{-37 \pm 9}$
& $\mathbf{+15.9\%}$
&
HumanoidStandup-v5
& 17
& $148902 \pm 7799$
& $\mathbf{165067 \pm 32322}$
& $\mathbf{+10.9\%}$
\\

Reacher-v5
& 2
& $-3.2 \pm 0.7$
& $\mathbf{-3 \pm 0.1}$
& $\mathbf{+6.3\%}$
&
Humanoid-v5
& 17
& $2475 \pm 909$
& $\mathbf{3060 \pm 724}$
& $\mathbf{+23.6\%}$
\\

\hline
\end{tabular}%
}
\end{table*}










\begin{figure*}[t]
    \centering
    \begin{subfigure}[t]{0.32\textwidth}
        \centering
        \includegraphics[width=\linewidth]{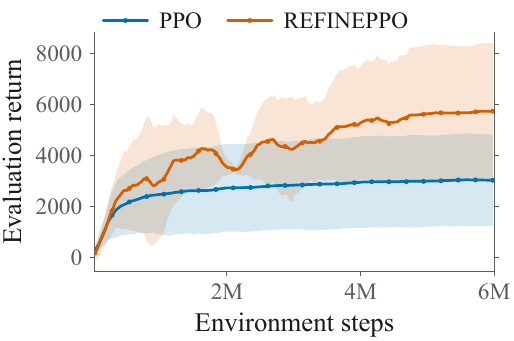}
        \caption{HalfCheetah-v5}
        \label{fig:HalfCheetah-v5}
    \end{subfigure}
    \hfill
    \begin{subfigure}[t]{0.32\textwidth}
        \centering
        \includegraphics[width=\linewidth]{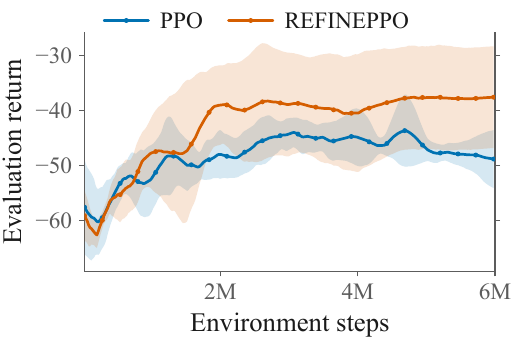}
        \caption{Pusher-v5}
        \label{fig:Pusher-v5}
    \end{subfigure}
    \hfill
    \begin{subfigure}[t]{0.32\textwidth}
        \centering
        \includegraphics[width=\linewidth]{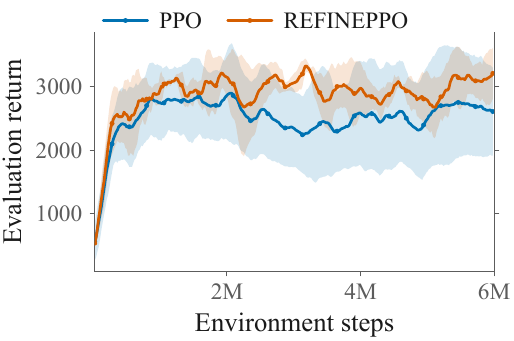}
        \caption{Hopper-v5}
        \label{fig:Hopper-v5}
    \end{subfigure}

    \vspace{0.5em}

    \begin{subfigure}[t]{0.32\textwidth}
        \centering
        \includegraphics[width=\linewidth]{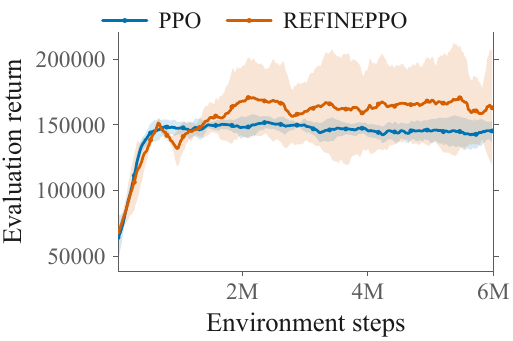}
        \caption{HumanoidStandup-v5}
        \label{fig:HumanoidStandup-v5}
    \end{subfigure}
    \hfill
    \begin{subfigure}[t]{0.32\textwidth}
        \centering
        \includegraphics[width=\linewidth]{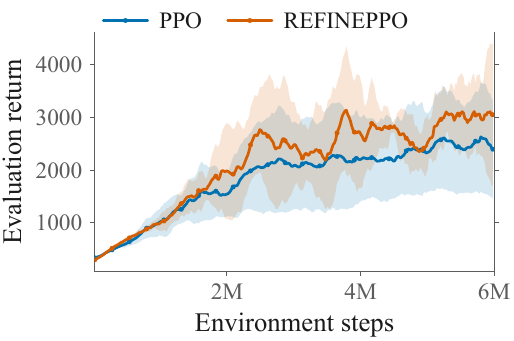}
        \caption{Humanoid-v5}
        \label{fig:Humanoid-v5}
    \end{subfigure}
    \hfill
    \begin{subfigure}[t]{0.32\textwidth}
        \centering
        \includegraphics[width=\linewidth]{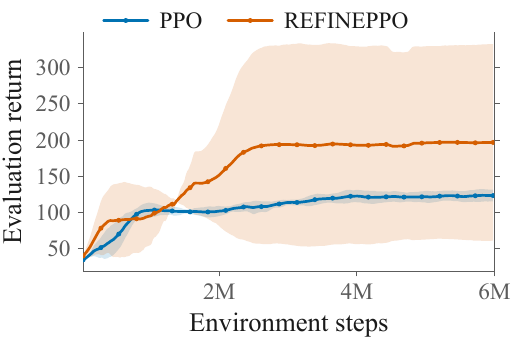}
        \caption{Swimmer-v5}
        \label{fig:Swimmer-v5}
    \end{subfigure}

    \caption{Evaluation returns throughout training for PPO and {\method} 
    across six representative environments. Solid lines show the mean 
    evaluation return across five independent seeds, and shaded regions 
    indicate variability across seeds.}
    \label{fig:plots}
\end{figure*}

We conduct experiments to evaluate the effectiveness of {{\method}} and to better understand the design choices underlying iterative action refinement. In particular, we aim to answer the following five research questions:

\begin{enumerate}
\item \textbf{RQ1: Performance.} Does {{\method}} achieve better performance than standard PPO?
\item \textbf{RQ2: Sample Efficiency.} Is {{\method}} more sample-efficient than standard PPO?
\item \textbf{RQ3: Mechanism.} What makes {{\method}} work?
\item \textbf{RQ4: Refinement Depth.} How does the refinement depth $K$ affect performance?
\item \textbf{RQ5: Refinement Schedule.} How does the iterative update schedule affect performance?
\end{enumerate}

In the following sections, we take a deep dive on these.

\vspace{-0.2cm}
\subsection{RQ1: Performance}

We first compare the final performance of {\method} against standard PPO across the evaluation environments. Table~\ref{tab:performance} reports the mean evaluation return over five independent seeds. Overall, {\method} outperforms PPO in many of the environments while remaining competitive in the others, providing strong evidence for the utility of {\method}. The results also suggest that the benefits of {\method} may become more pronounced as the dimensionality of the action space increases. One possible explanation is that higher-dimensional action spaces make  the policy to coordinate a larger number of action dimensions which {\method} is better designed to handle. 

\begin{figure*}[t]
    \centering

    \begin{subfigure}[t]{0.48\textwidth}
        \centering
        \includegraphics[width=\linewidth]
            {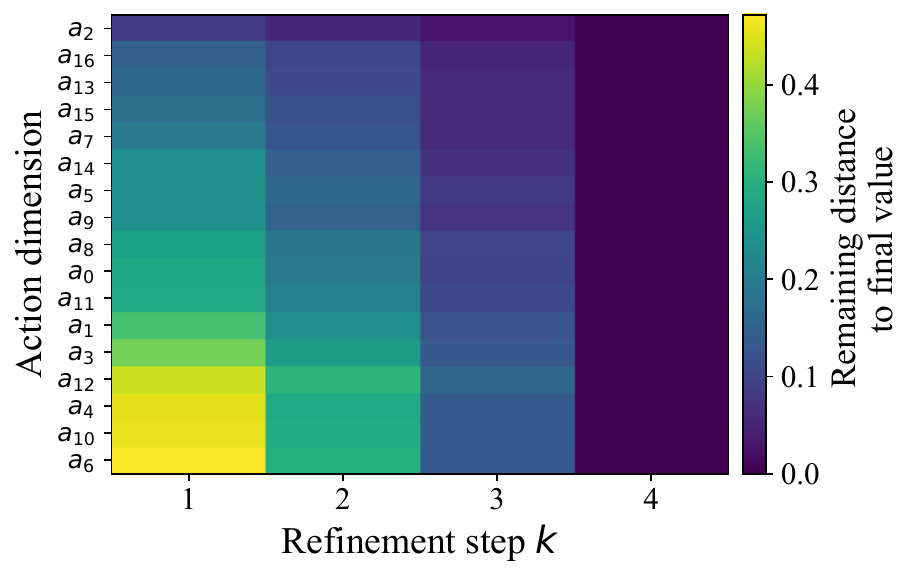}
        \caption{
        \textbf{Remaining refinement by action component.}
Each heatmap entry $(j,k)$ shows the normalized remaining distance
$e_{k,j}$ between intermediate action component $m_{k,j}$ and its
terminal value $m_{K,j}$, as defined in Eq.~\ref{eq:settling_error}.
Smaller values indicate that the component is already closer to its
terminal value. Components are ordered such that those
requiring the least subsequent refinement shown first.
}
        \label{fig:component_settling}
    \end{subfigure}
    \hfill
    \begin{subfigure}[t]{0.48\textwidth}
        \centering
        \includegraphics[width=\linewidth]
            {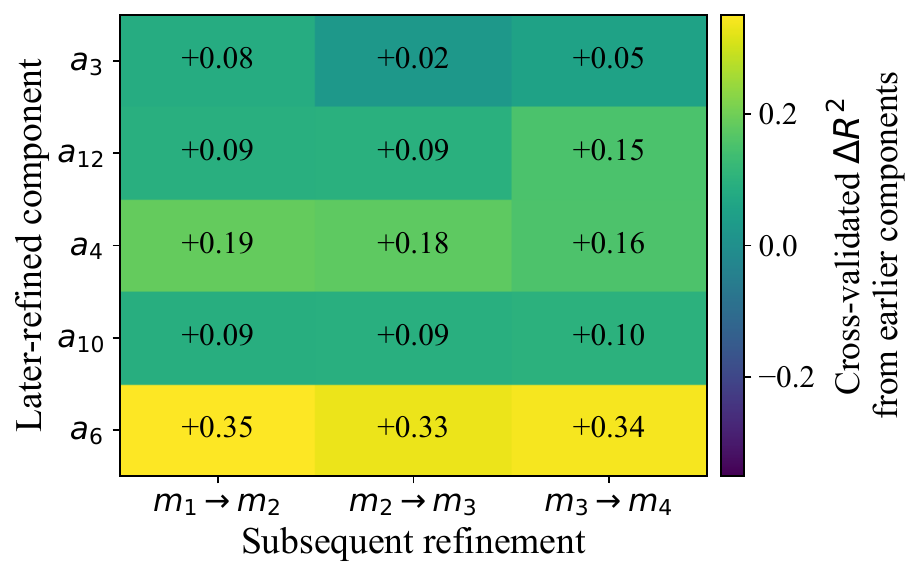}
        \caption{
        \textbf{Earlier components inform later refinements.}
        Increase in five-fold cross-validated $R^2$ when
        earlier-established components are added to a baseline that
        predicts the next correction using only the target component's
        current value. Positive $\Delta R^2$ indicates additional
        predictive information from the earlier components.
        }
        \label{fig:early_to_late}
    \end{subfigure}

    \caption{
    \textbf{Empirical evidence for progressive action construction in
    a trained {\rar} policy on Humanoid Standup.}
    \textbf{(a)} Some action components require substantially less
    subsequent refinement than others.
    \textbf{(b)} Intermediate values of earlier-established components
    provide information about subsequent corrections to components that
    continue to be refined.
    Together, the results are consistent with progressive,
    cross-component action refinement.
    }
    \label{fig:progressive_refinement}
\end{figure*}

\subsection{RQ2: Sample Efficiency}

We next examine how quickly the two methods learn as a function of environment interaction. Figure~\ref{fig:plots} shows the evaluation returns throughout training on six representative environments. {\method} learns more rapidly than PPO on all six representative environments, with the clearest improvement on Swimmer-v5, where the performance gap emerges early and persists throughout training. In contrast, Pusher-v5 and Humanoid-v5 exhibit similar learning dynamics for {\method} and PPO, with their learning curves largely tracking each other throughout training. Overall, these results suggest that {\method} can improve sample efficiency in several environments while preserving learning efficiency comparable to PPO in others.

\subsection{Empirical Evidence for the Refinement Mechanism}
\label{sec:refinement_evidence}

We next examine whether {\rar} exhibits the progressive action-construction behavior motivated with a example in Section~\ref{sec:understanding}. In particular, we ask whether some action components become established earlier in the refinement trajectory, and whether these earlier components contain information useful for predicting subsequent refinements of other components. We test this hypothesis with a {\method} trained HumanoidStandup-v5 rollout.

\paragraph{When are action components established?}
For each action dimension $j$, we measure its remaining distance from the
terminal proposal after refinement step $k$:
\begin{equation}
    e_{k,j}
    =
    \frac{
        \mathbb{E}_s[|m_{K,j}(s)-m_{k,j}(s)|]
    }{
        \mathbb{E}_s[|m_{K,j}(s)|] + \epsilon
    }
    \label{eq:settling_error}
\end{equation}
We compute this quantity over 1,000 intermediate action trajectories from a trained Humanoid Standup policy. A small $e_{k,j}$ indicates that component $j$ is already close to its terminal value at step $k$. We order action dimensions by $e_{1,j}$, so that dimensions requiring the least subsequent refinement appear first. Figure~\ref{fig:progressive_refinement}(a) shows that action components are not refined uniformly: some are already close to their terminal values after the first refinement, while others undergo substantially greater subsequent revision.

\paragraph{Do earlier components inform later refinements?} We next test whether components that become established earlier contain information about how later-refined components will subsequently change. Using the ordering from previous analysis, we select the five earliest components and the five components with the most remaining refinement. For each later refined component $j$ and transition $k\rightarrow k+1$, the prediction target is its next correction,

\begin{equation}
    \Delta_{k+1,j}=m_{k+1,j}-m_{k,j}
\end{equation}

We compare two ridge regressions. The baseline predicts $\Delta_{k+1,j}$ using only the component's current value $m_{k,j}$, while the augmented model additionally receives the current values of the five earlier-established components. We evaluate both models using five-fold cross-validation and report

\begin{equation}
    \Delta R^2
    =
    R^2_{\mathrm{augmented}}
    -
    R^2_{\mathrm{baseline}}
    \label{eq:cross_component_r2}
\end{equation}

Thus, positive $\Delta R^2$ indicates that earlier-established action components provide predictive information about the subsequent refinement of component $j$ beyond its own current value. Figure~\ref{fig:progressive_refinement}(b) reports this improvement for each later-refined component and refinement transition.

Together, the two analyses are consistent with progressive, coordinated action construction: some components require relatively little subsequent refinement, and their intermediate values contain information about how other components are refined at later steps.

\subsection{RQ3: Refinement Depth}

\begin{figure}
    \centering
    \includegraphics[width=0.75\linewidth]{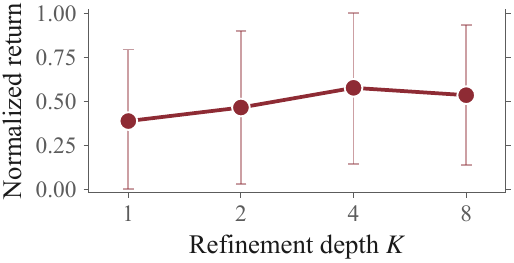}
    \caption{Refinement depth $K$ effect on normalized return across environments. Error bars indicate $\pm 1$ standard deviation.}
    \label{fig:RQ3}
\end{figure}

We next study the effect of refinement depth using $K\in\{1,2,4,8\}$. Figure~\ref{fig:RQ3} reports the min--max normalized mean return across the evaluation environments, with error bars indicating one standard deviation across environments. Performance generally improves as the refinement depth increases from $K=1$ to $K=4$, with $K=4$ achieving the highest average normalized return. Importantly, $K=1$ corresponds to the non-iterative boundary case: since $m_0=\mathbf{0}$, the policy directly maps the state and fixed initial proposal to an action. The improvement for $K>1$ therefore suggests that repeatedly refining an action proposal can improve policy performance.

The gains, however, do not increase monotonically with refinement depth, as performance slightly decreases from $K=4$ to $K=8$. This suggests that a moderate number of refinement steps is sufficient to capture most of the benefit, while additional refinement provides diminishing returns.

\subsection{RQ4: Refinement Schedule}

Finally, we investigate how the iterative update schedule influences refinement. Table~\ref{tab:RQ4} compares the Uniform ($c_k=1$), Inverse ($c_k=1/k$), and Average ($c_k=1/K$) schedules. We observe Average to be clearly providing the strongest overall performance across all the evaluated environments.

The schedules control how strongly each refinement step can modify the current action proposal and therefore induce different refinement dynamics. Uniform scaling allows every step to make a full residual correction, whereas the Inverse schedule progressively reduces the influence of later refinements. The Average schedule distributes a fixed total update scale across all $K$ steps. The differences in performance indicate that simply performing repeated refinement is not sufficient, how corrections are accumulated also matters. 

\begin{table}[t]
\centering
\caption{Effect of refinement schedule on \method{} performance across MuJoCo continuous-control environments. Results report the mean evaluation return for the Average, Uniform, and Inverse refinement schedules.}
\label{tab:RQ4}

\begin{tabular}{|l|c|c|c|}
\hline
\textbf{Environment} & \textbf{Average} & \textbf{Uniform} & \textbf{Inverse} \\
\hline
Ant-v5
& $\mathbf{4370.4}$
& $3298.8$
& $2892.9$
\\

HalfCheetah-v5
& $\mathbf{5685.9}$
& $1361.2$
& $1713.1$
\\

Hopper-v5
& $\mathbf{3120.4}$
& $2897.0$
& $2708.5$
\\

Humanoid-v5
& $\mathbf{3060.5}$
& $1345.5$
& $1024.5$
\\

HumanoidStandup-v5
& $\mathbf{165066.3}$
& $149432.5$
& $147769.1$
\\

Pusher-v5
& $\mathbf{-37.8}$
& $-58.4$
& $-44.5$
\\

Reacher-v5
& $\mathbf{-3.0}$
& $-3.5$
& $-3.3$
\\

Swimmer-v5
& $\mathbf{196.6}$
& $138.4$
& $47.8$
\\

Walker2d-v5
& $\mathbf{5809.7}$
& $4788.4$
& $4530.5$
\\

\hline
\end{tabular}
\end{table}

\section{Conclusion and Future Work}

In this work, we introduced \emph{Iterative Action Refinement} (\rar), an iterative action-construction approach that allows a continuous-control policy to progressively refine its action before execution. Integrated with PPO, {\method} achieves improved or competitive performance across 14 continuous control tasks from classic control, Bos2D and MuJoCo environments, with faster learning in several environments. Our analysis further suggests that later refinement steps use intermediate action proposals to coordinate and improve action components. Overall, these results demonstrate that iterative action construction is a promising alternative to conventional one-pass policy prediction.

Future work can extend \rar{} to other actor--critic algorithms such as SAC and TD3 and evaluate whether its benefits generalize beyond PPO. Another promising direction is adaptive refinement, where the policy determines how many refinement steps are needed for each state. Finally, evaluating \rar{} on more complex control problems, including real-world robotic manipulations and other systems, could further establish the generality and practical value of iterative action refinement. Such extensions would further clarify when iterative refinement is most beneficial and how it can be scaled to more challenging decision-making settings.

\bibliographystyle{IEEEtran}
\bibliography{bib}

\end{document}